\documentclass[journal]{IEEEtran}
\usepackage{color}
\usepackage{caption}
\usepackage{times}
\usepackage{helvet}
\usepackage{courier}
\usepackage{amsmath}
\usepackage{amssymb}
\usepackage{amsthm}
\usepackage{graphicx}
\usepackage{array}
\usepackage{booktabs}
\usepackage{amsopn}
\usepackage{amsmath,bm}
\usepackage{algorithm}
\usepackage{algorithmic}
\usepackage{enumerate}

\newcommand{\eg}{\emph{e.g.}}

\newcommand{\ie}{\emph{i.e.}}
\newcommand{\etc}{\emph{etc.}}
\newcommand{\wrt}{w.r.t. }

\newtheorem{myTheo}{Theorem}

\graphicspath{{figs/}}
\begin{document}

\title{Parts for the Whole: \\The DCT Norm for Extreme Visual Recovery}

\author{Yunhe Wang, Chang Xu, Shan You, 
Dacheng Tao,~\IEEEmembership{Fellow,~IEEE} and Chao Xu
\IEEEcompsocitemizethanks{\IEEEcompsocthanksitem Y. Wang, C. Xu, S. You, and C. Xu are with the Key Laboratory of Machine Perception (Ministry of Education) and Cooprtative Medianet Innovation Center, School of Electronics Engineering and Computer Science, Peking University, Beijing 100871, P.R. China. E-mail: \{wangyunhe, xuchang, youshan\}@pku.edu.cn, xuchao@cis.pku.edu.cn.
\IEEEcompsocthanksitem D. Tao is with the Centre for Quantum Computation and Intelligent Systems, Faculty of Engineering and Information Technology, University of Technology at Sydney, Sydney, NSW 2007, Australia. E-mail: dacheng.tao@uts.edu.au.
}}

\markboth{ }
{\MakeLowercase{Wang \textit{et al.}}: Parts for the Whole: The DCT Norm for Extreme Visual Recovery}

\IEEEtitleabstractindextext{%
\begin{abstract}
Here we study the extreme visual recovery problem, in which over 90\% of pixel values in a given image are missing. Existing low rank-based algorithms are only effective for recovering data with at most 90\% missing values. Thus, we exploit visual data's smoothness property to help solve this challenging extreme visual recovery problem. Based on the Discrete Cosine Transformation (DCT), we propose a novel DCT norm that involves all pixels and produces smooth estimations in any view. Our theoretical analysis shows that the total variation (TV) norm, which only achieves local smoothness, is a special case of the proposed DCT norm. We also develop a new visual recovery algorithm by minimizing the DCT and nuclear norms to achieve a more visually pleasing estimation. Experimental results on a benchmark image dataset demonstrate that the proposed approach is superior to state-of-the-art methods in terms of peak signal-to-noise ratio and structural similarity.
\end{abstract}

}

\maketitle

\IEEEdisplaynontitleabstractindextext

\IEEEpeerreviewmaketitle

\section{Introduction}
\IEEEPARstart{V}{isual} data can be corrupted due to sensory noise or interferential outliers during data acquisition.  A fraction of image pixels could be missing sometimes, if the conditions deteriorate further, missing pixels will constantly increase, \eg, terribly-damaged images due to unstable online transmission, cameras covered by noises, photographs overexposed accidently, and outdoor pictures taken behind a screen window. For satisfactory visual recognition, detection, and tracking, corrupted visual data must be recovered during pre-processing~\cite{ijcai13,nuclearnorm,softthreshold,tensor13,MC10,ijcai15}

There has recently been a surge in low rank-based matrix completion methods for visual recovery. Given a matrix $M\in \mathcal{R}^{n \times m}$, where $n$ and $m$ are the width and height of $M$, respectively, ${M_{i,j}=0,\ (i,j) \notin \Omega}$ denote the observed data while the others represent missing data. In general, a low rank matrix $\hat{M}$ can be discovered to approximately represent matrix $M$. The rank of $\hat{M}$ is usually assumed to be lower than any of its two dimensionalities, \ie, $rank(\hat{M})\ll \min(n,m)$. 

Although rank minimization provides an approach for recovering the missing observations, it is computationally intractable (NP-hard and non-convex). Under broad conditions, \cite{nuclearnorm} reported that the rank function minimization can be replaced by the trace norm $||X||_* = \sum_k \sigma_k\left(X\right)$ minimization, where $\sigma_k$ is the $k$-$th$ maximum singular value of $X$. Soft-thresholding methods~\cite{softthreshold} are often employed to solve optimization problems with this trace norm regularization. However, trace norms have been improved to better investigate the low rank constraint. For example, \cite{tnuclearnorm} designed a truncated nuclear norm more suitable for matrix completion and~\cite{WNNM} proposed a weighted nuclear norm minimization for image denoising.

\begin{figure}[t]
\begin{center}
   \includegraphics[width=0.95\linewidth]{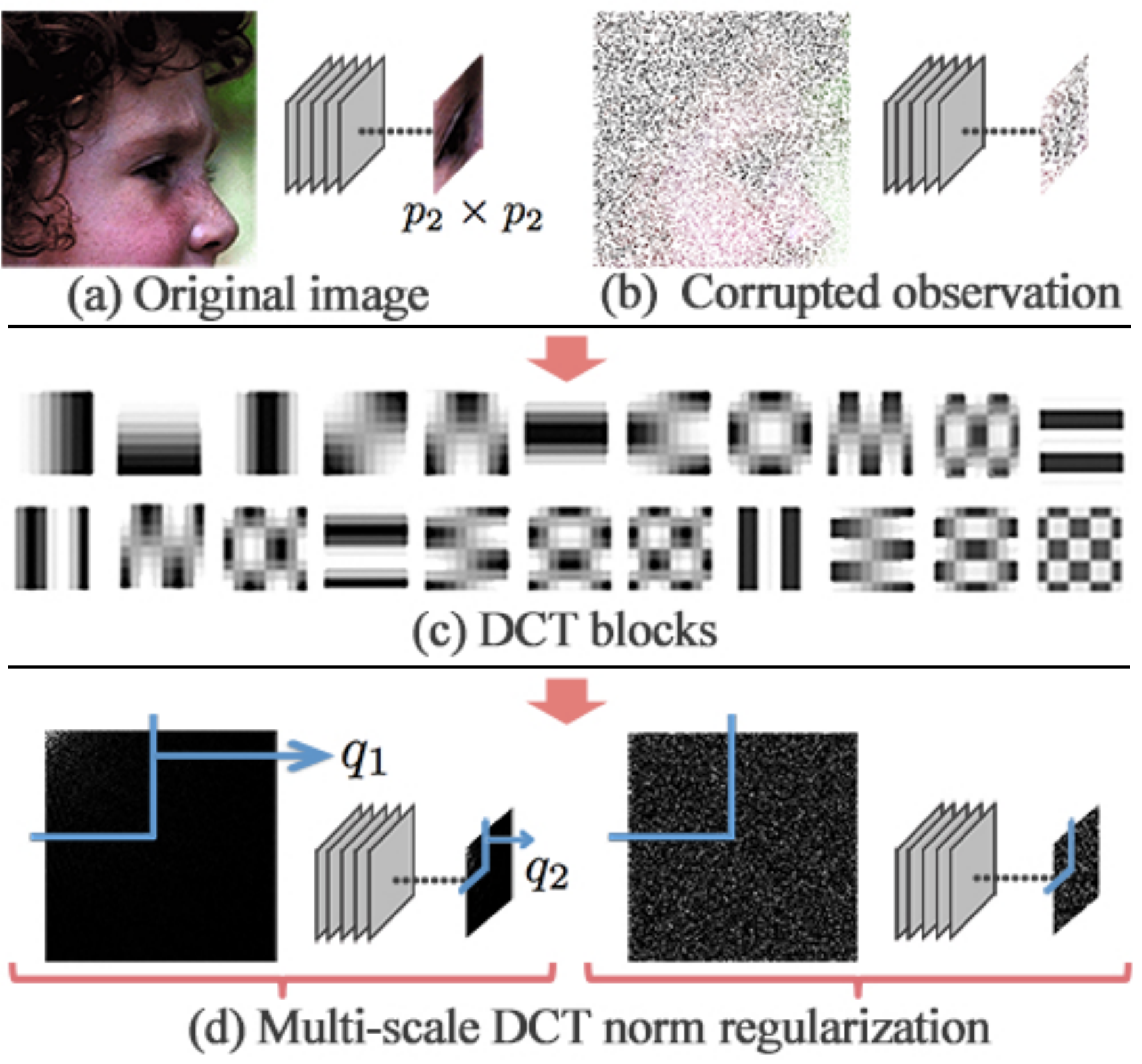}
\end{center}
   \caption{A schematic of the calculation of the proposed DCT norm at two scales. The input image is first densely divided into several $p_i \times p_i$ patches that are then converted to the frequency domain. The DCT norm is returned by accumulating squared values of high-frequency coefficients whose positions are larger than $q_i$ in all the patches.}
\label{Fig:Intro}
\end{figure}

Existing rank minimizing techniques are effective and have delivered promising performance in many visual recovery problems. However, most of these algorithms have only been evaluated for data with at most 80\% missing values~\cite{tensor13}. They may not, therefore, be applicable to data with extremely high numbers of missing values (\eg, 95\%) that occur when data collection devices fail or there is damage to the transmission medium. This extreme visual recovery problem with a very large number of missing pixel values is, therefore, very challenging due to the conflict between the small amount of observed data and the tremendous amount of data that needs to be recovered.

As well as the low rank assumption, smoothness is another important property carried by visual data. The total variation (TV) norm~\cite{TVnorm} is an effective way to exploit smoothness by calculating the differences between neighboring pixels and has been widely used in image processing. However, minimizing local differences via the TV norm risks over-smoothing the image detail and texture~\cite{NLM}.

We consider employing the Discrete Cosine Transformation (DCT~\cite{DCT}) to devise a new smoothness regularizer called the ``DCT norm''. By constraining the frequency coefficients of the data obtained via DCT, the DCT norm flexibly adjusts the degree of smoothness over the data (Fig.\ref{Fig:Intro}). To investigate the local smoothness and multi-scale properties of visual data, we extend the DCT norm to derive local and multi-scale versions, respectively. Our theoretical analysis shows that the DCT norm includes the TV norm as a special case. By combining the classical rank minimization principle with the proposed DCT norm minimization, the resulting model can simultaneously exploit the low rank and smoothness properties of visual data for extreme visual recovery. Experimental results on real-world datasets demonstrate that the proposed DCT norm is highly effective and that the low rank and smoothness issues need to be integrated for successful extreme visual recovery.

\section{Related works}
We first briefly introduce related works on visual completion and TV norm minimization.

Nuclear norm minimization-based matrix completion~\cite{nuclearnorm} is formulated as 
\begin{equation}
\hat{X} = \arg\min_X ||X||_*,\quad s.t.\ X_\Omega = M_\Omega,
\label{Fcn:lowrank}
\end{equation}
where $\hat{X}$ is the optimal estimation of $M$. The truncated nuclear norm minimization problem~\cite{tnuclearnorm}  is defined as: 
\begin{equation}
\hat{X} = \arg\min_X ||X||_r,\quad s.t.\ X_\Omega = M_\Omega,
\label{Fcn:tlowrank}
\end{equation}
where $||X||_r = \sum_{i=r+1}^{\min(n,m)}\sigma_i(X)$ is the sum of $\min(n,m)-r$ minimum singular values.
 
It is widely known that the TV norm~\cite{TVnorm} is an efficient smoothness regularization that accumulates all the gradients of a given image $X$: 
\begin{equation}
||X||_{TV} = \sum_{i,j} \sqrt{|X_{i+1,j}-X_{i,j}|^2+|X_{i,j+1}-X_{i,j}|^2},
\label{Fcn:TVnorm}
\end{equation}
where $i$ and $j$ denote the vertical and horizontal positions of $X$, respectively. Since the TV norm accumulates gradient modules of the entire image, minimizing the TV norm can result in a smooth estimation: 
\begin{equation}
\hat{X} = \arg\min_X ||X||_{TV}, \quad s.t.\ X_\Omega = M_\Omega.
\label{Fcn:minTV}
\end{equation}

\cite{jointTV} proposed to simultaneously use the TV and nuclear norms for image recovery: 
\begin{equation}
\hat{X} = \arg\min_X||X||_*+\lambda||X||_{TV}, \quad s.t.\ X_\Omega = M_\Omega,
\label{Fcn:joint}
\end{equation}
where $\lambda$ balances the two norms. To the best of our knowledge, Fcn.\ref{Fcn:joint} is the first algorithm to integrate smoothness and nuclear norm regularization.

Since the $||\cdot||_{TV}$ is isotropic and not differentiable, an anisotropic version is proposed in~\cite{aTVnorm} that is easier to minimize: 
\begin{equation}
||X||_{anisoTV}=\sum_{i,j}|X_{i+1,j}-X_{i,j}|+|X_{i,j+1}-X_{i,j}|.
\end{equation}

Additionally, in~\cite{LTV}, a modified linear total variation was defined as: 
\begin{equation}
||X||_{LTV} = \sum_{i,j}|X_{i+1,j}-X_{i,j}|^2+|X_{i,j+1}-X_{i,j}|^2,
\end{equation}
which leads to a smooth low-rank matrix completion problem: 
\begin{equation}
\hat{X} = \arg\min_X||X||_*+\lambda||X||_{LTV}, \quad s.t.\ X_\Omega = M_\Omega.
\label{Fcn:LTVnuclear}
\end{equation}

An ADMM-like optimization scheme~\cite{ADMM} can be adopted to solve Fcn.\ref{Fcn:LTVnuclear}. However, the traditional TV norm only can guarantee an estimation presenting a locally smooth visualization, when in reality a natural image should be smoothed at every scale. In the next section, we propose a multi-scale DCT norm in the frequency domain.

\section{The DCT norm}
A natural image has smooth regions in the spatial domain. Furthermore, there is less high-frequency information than low-frequency information in the image's frequency domain~\cite{DCTdistribution}. This section presents a novel Discrete Cosine Transformation (DCT) norm for smoothing the objective variable, which provides advantages over the TV norm due to its linear and convex properties.

\subsection{The DCT for a 2D matrix}
DCT is widely used in image compression and is an approximate KL transformation~\cite{DCT}. For an arbitrary matrix $X$, its DCT coefficient matrix $\mathcal{C}$ is: 

\begin{equation}
\mathcal{C}_{X} = \mathcal{T}\left(X\right) = CXC^{\mathbf T} = \left(C \otimes C\right)vec(X),
\label{Fcn:DCT}
\end{equation}
where $\otimes$ is the Kronecker product, $vec(\cdot)$ denotes the vectorization, and $C$ denotes the transformation matrix with the same size as $X$: 

\begin{equation}
C_{ij} = \alpha\left(i\right) cos\left[\frac{(j+0.5)\pi}{N}i\right],
\end{equation}
and
\begin{equation}
\alpha\left(i\right) = \left\{
\begin{aligned}
& \sqrt{1/N}, \ \ i =0\\
& \sqrt{2/N}, \ \ else.
\end{aligned}
\right.
\end{equation}

$\mathcal{C}$ has the same dimensionality with $X$, where $\mathcal{C}_{0,0}$ is the DC (direct current) coefficient which only consists of the overall illumination information~\cite{DC0}, the other coefficients in  $\mathcal{C}$ are AC (alternating current) components denote the energies of every frequency levels, \ie the weights of the DCT blocks as showed in Fig.\ref{Fig:Intro}(c). Additionally, it is instructive to note that DCT is a linear lossless transformation and the original data can be restored by $X=\mathcal{T}^{-1}\left(\mathcal{C}_X\right)=\left(C \otimes C\right)^T \mathcal{C}_X$. DCT is also applied to a variety of computer vision tasks, such as image denoising~\cite{BM3D} and image representation~\cite{DCTface}.

Since Fcn.\ref{Fcn:DCT} involves all elements in $X$ we can access the overall structural information of the entire matrix. The DCT and its gradient can be quickly calculated using linear transformations such that the proposed DCT norm-based optimization problems can be efficiently and easily solved.

\subsection{The DCT norm for smoothing}
\textbf{The global smooth DCT norm.} Neighboring pixels in a natural image are generally significantly correlated. On the other hand, the abnormal signal (\eg, noise, missing values, \etc) can be seen as a set of external data subject to an i.i.d. distribution. Hence, the frequency distributions of natural images and the abnormal signals are distinct (Fig.\ref{Fig:Intro}). The high-frequency information of the original image is much lower than that of the corrupted observation $M$. Based on this observation, we design a DCT norm in the frequency domain: 

\begin{equation}
\begin{aligned}
||X||_{DCT}^q &= ||S_q*\left(CXC^{\mathbf T}\right)||_F^2\\
&= ||vec(S_q)*\left(C\otimes C vec(X)\right)||^2,
\end{aligned}
\label{Fcn:DCTglobal}
\end{equation}
where $*$ denotes the Hadamard product, $||\cdot||_F$ denotes the Frobenius norm, and $S_c$ is a selection mask with parameter $c$ denoting the cut-off position: 
\begin{equation}
S_q = \left\{ \begin{array}{l}
S_{ij} =  1, \quad i\leq q\ \&\ j\leq q \\
S_{ij} =  0, \quad  otherwise.
\end{array} \right.
\end{equation}

The smoothing-oriented visual recovery problem can thus be formulated as: 
\begin{equation}
\hat{X} = \arg\min_X ||X||_{DCT}^q, \quad s.t.\ X_\Omega = M_\Omega.
\label{Fcn:globalsmooth}
\end{equation}

Although we can obtain a smooth estimation by solving the above problem, the optimal solution of Fcn.\ref{Fcn:globalsmooth} will have some deformations as shown in Fig.\ref{Fig:exp}. This is due to some remaining frequency coefficients with positions lower than $q$ still needing to be discarded since they were difficult to estimate. A local smoothness regularization can eliminate these deformations. Therefore, we expand the global smoothed DCT norm to a more comprehensive model that can also represent the local smoothness of the given image.

\noindent \textbf{The locally smooth DCT norm.} Inspired by non-local denoising methods~\cite{NLM,WNNM}, we divide the corrupted observation into several small patches or the locally smooth DCT norm:
\begin{equation}
\begin{aligned}
||X||_{DCT}^{p,q} &= \sum_l ||S_{p,q}*\left(C_p x_p^{(l)} C_p^{\mathbf T}\right)||_F^2\\
&= ||\mathbf{S}_{p,q}*\left(C_p \otimes C_p \mathbf{X}_p\right)||_F^2,
\end{aligned}
\label{Fcn:DCTlocal}
\end{equation}
where $p$ is the scale parameter, $||X||_{DCT}^{p,q}$ will be equivalent to Fcn.\ref{Fcn:DCTglobal} given $p = N$, $S_{p,q}$ and $C_p$ denote the mask and the transformation matrix generated according to p, respectively,  $x_p^{(l)}$ is a $p \times p$ matrix denoting the $l$-$th$ patch extracted from $X$, $\mathbf{X}_p = \left[vec(x_p^{(1)}), vec(x_p^{(2)}), ..., vec(x_p^{(l)})\right]$ stacks $x_p^{(l)}$ into a matrix, and $\mathbf{S}_{p,q} = [vec(S_{p,q}),...,vec(S_{p,q})]$, which has the same dimensionality to that of $\mathbf{X}$. Patches extracted at every pixel are overlapping. For an $N\times M$ image, there are $(N-p+1)\times(M-p+1)$ patches with size $p\times p$. 
 
The resulting optimization function \wrt the local DCT norm is: 
\begin{equation}
\hat{X} = \arg\min_X ||X||_{DCT}^{p,q}, \quad s.t.\ X_\Omega = M_\Omega.
\label{Fcn:localmin}
\end{equation}

\noindent \textbf{The multi-scale DCT norm.} Both the locally smooth DCT norm and the global smooth DCT norm have pros and cons and can be further refined. Inspired by research on the local descriptor~\cite{SIFT,multiregion}, detecting and describing a key point based on multiple scales can capture more geometric information. Hence, we propose to integrate local and global smoothness by integrating DCT norms from multiple image scales:
\begin{equation}
\hat{X} = \arg\min_X \sum_{i=1}^{s} ||X||_{DCT}^{p_i,q_i}, \quad s.t.\ X_\Omega = M_\Omega.
\label{Fcn:multiscale}
\end{equation}
where $s$ indicates the number of scales. If $s = 1$ and $p = N$, Fcn.\ref{Fcn:multiscale} will focus on global smoothness. Given $s = 1$ and $p = 2$, Fcn.\ref{Fcn:multiscale} will be reduced to the local DCT norm minimization problem in Fcn.\ref{Fcn:DCTlocal}. 

\subsection{Relationship to the TV Norm}
The TV norm is an efficient smoothing tool (see Fcn.\ref{Fcn:TVnorm}). The following theorem suggests that the TV norm can be regarded as a special case of DCT norm, thereby demonstrating the superiority of the proposed DCT norm. 

\begin{myTheo}
Given $X \in \mathbb{R}^{n \times m}$, the optimal solution of Fcn.\ref{Fcn:minTV} with the linear TV norm is exactly that of the DNM problem in Fcn.\ref{Fcn:localmin} with $p$ = 2 and $q$ = 1.
\end{myTheo}
\begin{proof}
For any $2\times 2$ patch $x=[x_{0,0},x_{1,0},x_{0,1},x_{1,1}]^\mathbf{T}$ in $\mathbf{X}$, its DCT coefficient matrix is $\mathcal{C} = \mathcal{T}(x)$ and the local DCT norm is
\begin{equation}
||x||_{DCT}^{2,1} = \mathcal{C}_{0,1}^2+\mathcal{C}_{1,0}^2+\mathcal{C}_{1,1}^2,
\end{equation}
where
\begin{equation}
\begin{aligned}
\mathcal{C}_{0,1} &= \frac{1}{2}\left((x_{0,0}-x_{0,1})+(x_{1,0}-x_{1,1}) \right),\\
\mathcal{C}_{1,0} &= \frac{1}{2}\left((x_{0,0}-x_{1,0})+(x_{0,1}-x_{1,1}) \right),\\
\mathcal{C}_{1,1} &= \frac{1}{2}\left((x_{0,0}+x_{1,1})-(x_{0,1}+x_{1,0}) \right),
\end{aligned}
\end{equation}
and the TV norm of $X$ is calculated as:
\begin{equation}
\begin{aligned}
||x||_{LTV} &= (x_{00}-x_{01})^2+(x_{00}-x_{10})^2\\
&+(x_{10}-x_{11})^2+(x_{01}-x_{11})^2.
\end{aligned}
\end{equation}

Let $f(x) = 4||x||_{DCT}^{2,1}$, $e_1=[1,1,-1,-1]^\mathbf{T}$, $e_2=[1,-1,1,-1]^\mathbf{T}$, and $e_3=[1,-1,-1,1]^\mathbf{T}$. We have
\begin{equation}
\begin{aligned}
f(x)&=\left(e_1^\mathbf{T}x\right)^2+\left(e_2^\mathbf{T}x\right)^2+\left(e_3^\mathbf{T}x\right)^2\\
&=x^\mathbf{T}e_1e_1^\mathbf{T}x+x^\mathbf{T}e_2e_2^\mathbf{T}x+x^\mathbf{T}e_3e_3^\mathbf{T}x\\
&=\sum_{i=1}^3 x^\mathbf{T}e_ie_i^\mathbf{T}x,
\end{aligned}
\end{equation} 

Similarly, let $g(x) = ||x||_{LTV}$, $d_1=[1,0,-1,0]^\mathbf{T}$, $d_2=[1,-1,0,0]^\mathbf{T}$, $d_3=[0,1,0,-1]^\mathbf{T}$, and $d_4=[0,0,1,-1]^\mathbf{T}$. We have
\begin{equation}
g(x)=\sum_{i=1}^4 x^\mathbf{T}d_id_i^\mathbf{T}x,
\end{equation} 

Consider two minimizations $\min f(x)$ and $\min g(x)$. Their optimal solutions satisfy:
\begin{equation}
\begin{aligned}
\triangledown_x f(x)&=2\sum_{i=1}^3e_ie_i^\mathbf{T}x= \mathcal{E}x=0\\
\triangledown_x g(x)&=2\sum_{i=1}^4d_id_i^\mathbf{T}x= \mathcal{D}x=0,
\end{aligned}
\label{Fcn:nullspace}
\end{equation}
where $\mathcal{E} = 2\sum_{i=1}^3e_ie_i^\mathbf{T}$, $\mathcal{D}=2\sum_{i=1}^4d_id_i^\mathbf{T}$ and their null spaces are equivalent $\mathcal{N}(\mathcal{E}) = \mathcal{N}(\mathcal{D})$. 

Furthermore, it is easy to demonstrate the null spaces of $\mathcal{E}$ and $\mathcal{D}$ are identical for any $N \times M$ image by formatting its $\mathcal{E}$ and $\mathcal{D}$ according to Fcn.~\ref{Fcn:nullspace}. Hence, the optimal solution of DNM is equal to that of the TV norm minimization.
\end{proof}

According to Theorem 1, we conclude that $||x||_{DCT}^{2,1}$ is equivalent to the TV norm. Moreover, we suggest that the proposed approach in Fcn.\ref{Fcn:multiscale} has benefits over that of the TV norm since there is flexible to control over the degree of smoothness by operating at multiple scales.

\begin{figure}[h]
\centering
 \includegraphics[width=0.95\linewidth]{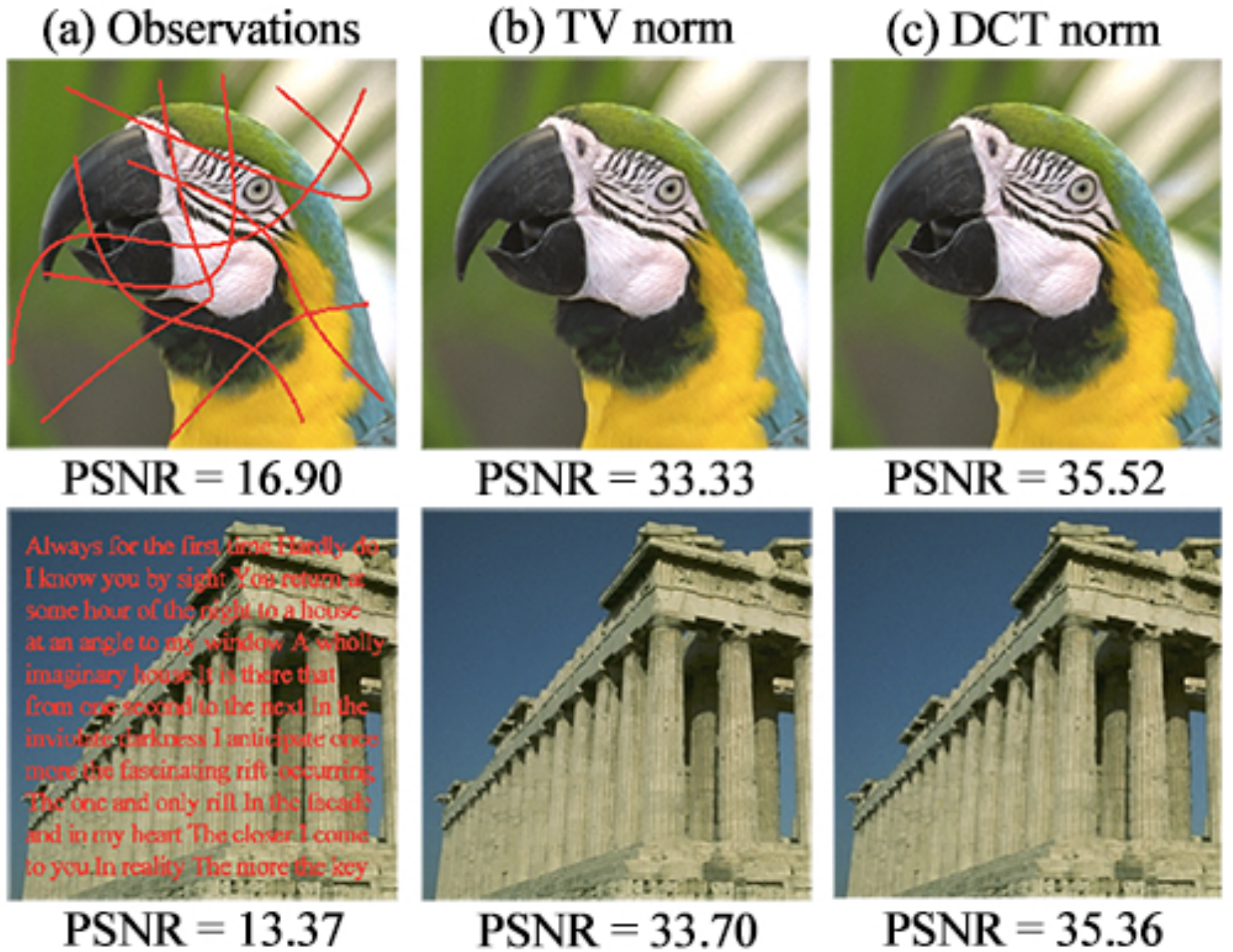}
\caption{A comparison of the DCT norm and TV norm in a removal problem. Left to right show corrupted observations and reconstructed images using the TV norm and DCT norm, respectively. }
\label{Fig:inpainting}
\end{figure}

The TV norm is noted for its strong capacity for image inpainting, \eg, text removal~\cite{tnuclearnorm}. This is a challenging task, because the pixels covered by the text are not randomly distributed. Hence, the optimization function with a regularization that encourages global smoothness might be expected to perform well. Since color images have three channels (\ie, red, green, and blue), for fair comparison with the TV norm we simply deal with each channel separately and combine them to obtain the final result. As illustrated in Fig.\ref{Fig:inpainting}, the visual recovery results obtained with the proposed DCT norm are superior to those of the TV norm. The results show that the proposed DCT norm outperforms the TV norm because it takes all pixels into consideration. Here, we use a two-scale DCT norm with $p_1 = 2$, $p_2 = 512$ (image sizes), $q_1 = 1$, and $q_2 = 256$. Obviously, the peak signal-to-noise ratio (PSNR) values of the proposed DCT norm are about 2\emph{dB} higher than that of the TV norm.
 
\section{DCT norm for visual recovery}
Image completion aims to restore an original image $X$ from its corrupted observation $M$ with an observed region $\Omega$. Recovering missing values in a matrix with limited observed information has recently attracted considerable interest~\cite{tensor13,MC10,tnuclearnorm}.

This problem is commonly addressed with inpainting~\cite{tnuclearnorm} or denoising~\cite{WNNM} methods, especially when the missing ratio is not too high ($\leq 80\%$). The non-local means [et al.2005] and its variations~\cite{WNNM,SAIST} are the current state-of-the-art techniques and exploit the self-similarity characteristic of images. However, when the observation is very corrupt, \eg, the missing ratio is higher than $90\%$ (see Fig.\ref{Fig:exp}), non-local based algorithms are less useful since they cannot find similar patches in the given image.

\begin{table*}[t]
\begin{minipage}{0.5\textwidth}
\centering
\normalsize
\caption{PSNR Comparison (Unit: $dB$) }
\label{Tab:PSNR}
\vspace{0.7em}
\begin{tabular}{c||c|c|c|c|c}
\hline   
$\phi$ & SVT & TNNR & LTVNN & SAIST & DNM \\
\hline \hline
90\% & 19.299 & 19.947 & 26.134 & 25.928 & \textbf{28.865}\\
\hline
95\% & 16.984 & 14.125 &  23.408 & 23.682 & \textbf{26.182}\\
\hline
98\% & 15.149 & 7.543 & 20.289 & 21.301 & \textbf{23.377}\\
\hline
99\% & 14.313 & 5.458 & 18.245 & 19.605 & \textbf{21.704}\\
\hline
\end{tabular}
\end{minipage}
\begin{minipage}{0.5\textwidth}
\centering
\normalsize
\caption{SSIM Comparison}
\label{Tab:SSIM}
\vspace{0.7em}
\begin{tabular}{c||c|c|c|c|c}
\hline
$\phi$ & SVT & TNNR & LTVNN & SAIST & DNM \\
\hline \hline
90\% & 0.7170 & 0.6970 & 0.9125 & 0.8876 & \textbf{0.9472}\\
\hline
95\% & 0.6300 & 0.4281 & 0.8647 & 0.8470 & \textbf{0.9161}\\
\hline
98\% &0.5134 & 0.1063 & 0.7917 & 0.7989 & \textbf{0.8712}\\
\hline
99\% &0.4291& 0.0544 & 0.7275  & 0.7624 & \textbf{0.8391}\\
\hline
\end{tabular}
\end{minipage}
\end{table*}

Base on the proposed multi-scale DCT norm and existing rank minimization techniques, we establish an efficient optimization problem that contains the various factors mentioned above: 
\begin{equation}
\begin{aligned}
\hat{X} =& \arg\min_X||X||_r+\sum_i \lambda_i ||X||_{DCT}^{p_i,q_i}\\
&+\frac{\gamma}{2}||\mathcal{P}_{\Omega}(X)-\mathcal{P}_{\Omega}(M)||_F^2,
\end{aligned}
\label{Fcn:ObjFcn}
\end{equation}
where $\lambda_i$ denotes the weighting parameter for the DCT norm $||X||_{DCT}^{p_i,q_i}$ of $X$ in the $i$-$th$ scale. $\gamma>0$ is a relaxation factor that converts the original problem into an unconstrained minimization~\cite{tnuclearnorm}. The relaxed problem can be solved using a gradient-based algorithm~\cite{Gradient}.

Fcn.\ref{Fcn:ObjFcn} is naturally designed for the recovery of gray images. However, it is easy to extend Fcn.\ref{Fcn:ObjFcn} to handle color images: 
\begin{equation}
\begin{aligned}
\hat{X} =& \sum_c \arg\min_X||X_{(c)}||_r+\sum_c \sum_i \lambda_i ||X_{(c)}||_{DCT}^{p_i,q_i}\\
&+\alpha\sum_c ||X_{(c)}||_{freq}+\frac{\gamma}{2}||\mathcal{P}_{\Omega}(X)-\mathcal{P}_{\Omega}(M)||_F^2,
\end{aligned}
\label{Fcn:colorObjFcn}
\end{equation}
where $X$ is an RGB image and $X_{(c)}$ denotes the $c$-$th$ channel of $X$, \ie,  $X_{(1)}$, $X_{(2)}$, $X_{(3)}$ corresponding to the red, green, and blue channels, respectively. An additional regularization with constant weight $\alpha$  is employed in Fcn.\ref{Fcn:colorObjFcn} to encourage DCT coefficients in any two channels of $X$ to be similar: 
\begin{equation}
||X_{(c)}||_{freq} = \frac{1}{4}\sum_{i\neq c}||S^c *C\left(X_{(c)}-X_{(i)}\right)C^\mathbf{T}||_F^2,
\end{equation}
where $S^c_{i,j} = 0\ if\ i = 0\ and\ j = 0,\ S^c_{i,j} = 1\ otherwise$ is a mask to remove the DC coefficient because the illumination information between channels is different. Although the value of $||X_{(c)}||_{freq}$ is equal to that calculated in the spatial domain after subtracting the average value of each channel, the calculation in the frequency domain will assign greater weights to low-frequency coefficients. This is important for constructing images of better visual quality. The gradient of the DCT norm is: 
\begin{equation}
\begin{aligned}
\mathcal{G}&\left(||X_{(c)}||_{DCT}\right) = \frac{\partial{\sum_i \lambda_i ||X_{(c)}||_{DCT}^{p_i,q_i}}}{\partial{X_{(c)}}}\\
& = \sum_i \lambda_i \mathcal{I} \left( \left(C_{p_i} \otimes C_{p_i}\right)^\mathbf{T} \left[ S_{p_i,q_i} * \left(C_{p_i}\otimes C_{p_i} \mathbf{X}_{p_i} \right) \right] \right),
\end{aligned}
\label{Fcn:DCTgradient}
\end{equation}
where $\mathcal{I}(\cdot)$ is an operation that recovers the stacking matrix $\mathbf{X}_{p_i}$ into the original image and $\mathcal{I}(\mathbf{X}_{p_i}) = X_{(c)}$. The gradient of $||X_{(c)}||_{freq}$ is 
\begin{equation}
\begin{aligned}
\mathcal{G}\left(||X_{(c)}||_{freq}\right) &= \frac{\partial{ ||X_{(c)}||_{freq}}}{\partial{X_{(c)}}}\\
 = \sum_{i\neq c} &C^\mathbf{T}\left(S^c *C\left(X_{(c)}-X_{(i)}\right)C^\mathbf{T}\right)C.
\end{aligned}
\end{equation}

An iterative TV norm~\cite{iterTV}  has also been proposed, and its optimal estimation is superior to that of directly solving the TV norm of the input observation. In this paper, we borrow this iterative strategy~\cite{SAIST,WNNM} and extend the DCT norm to an iterative method: 
\begin{equation}
M^{(k+1)}=M^{(k)}+\delta \mathcal{P}_{\Omega}\left(M-X^{(k)} \right),
\label{Fcn:iter}
\end{equation}
where $k$ is the iteration number and $\delta$ is a relaxation parameter (often set to 0.1). We describe the image completion method by exploiting the proposed multi-scale DCT norm as showed in Alg.\ref{Alg:Completion}.

\begin{algorithm}[h]
\caption{DCT norm minimization for visual recovery.}
\label{Alg:Completion}
\begin{algorithmic}[1]
\REQUIRE The corrupted observation $M^0 = M$, $\hat{X}^0 = M$ and the corresponding parameters;
\REPEAT
\STATE $X^k \leftarrow M^k$;
\REPEAT
\FOR{each scale $p_i$}
\STATE Divide $X^k$ into several $p_i \times p_i$ patches;
\STATE Accumulate the gradient of $X^k$ in the $p_i$ scale;
\ENDFOR
\STATE Calculate the gradient, except the nuclear norm;
\STATE Obtain the low-rank approximation;
\UNTIL{convergence}
\STATE $M^{(k+1)} \leftarrow  M^{(k)}+\delta \mathcal{P}_{\Omega} \left(M-X^{(k)} \right)$;
\UNTIL{$||M^{(k+1)}-X^{k}||_F \leq \epsilon$}
\ENSURE The estimated image $\hat{X}$;
\end{algorithmic}
\end{algorithm}

\section{Experiments}
\label{Sec:Exp}
\textbf{Experimental setup.} Experiments were carried on eight images~\cite{NCSR} widely used for evaluating the performance of image restoration algorithms (Fig.\ref{Fig:test}). This is a benchmark dataset widely used for evaluating the visual recovery performance, with a variety of scenes. 

\begin{figure}[h]
\begin{center}
\includegraphics[width=0.95\linewidth]{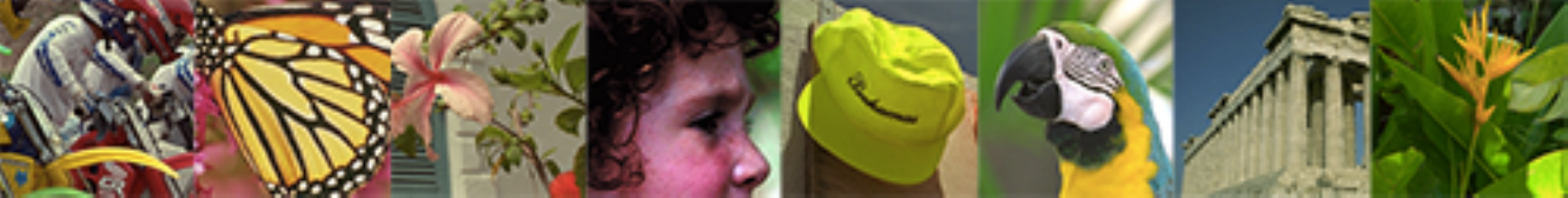}
\end{center}
\caption{The test images.}
\label{Fig:test}
\end{figure}

Observations were generated by randomly sampling a small proportion (ranging from 1\% to 10\%) of pixels from the images subject to a Gaussian distribution. We used two standard criteria to evaluate the recovery performance: PSNR and structural similarity (SSIM)~\cite{SSIM}.

\begin{figure*}[t]
\centering
 \includegraphics[width=1\linewidth]{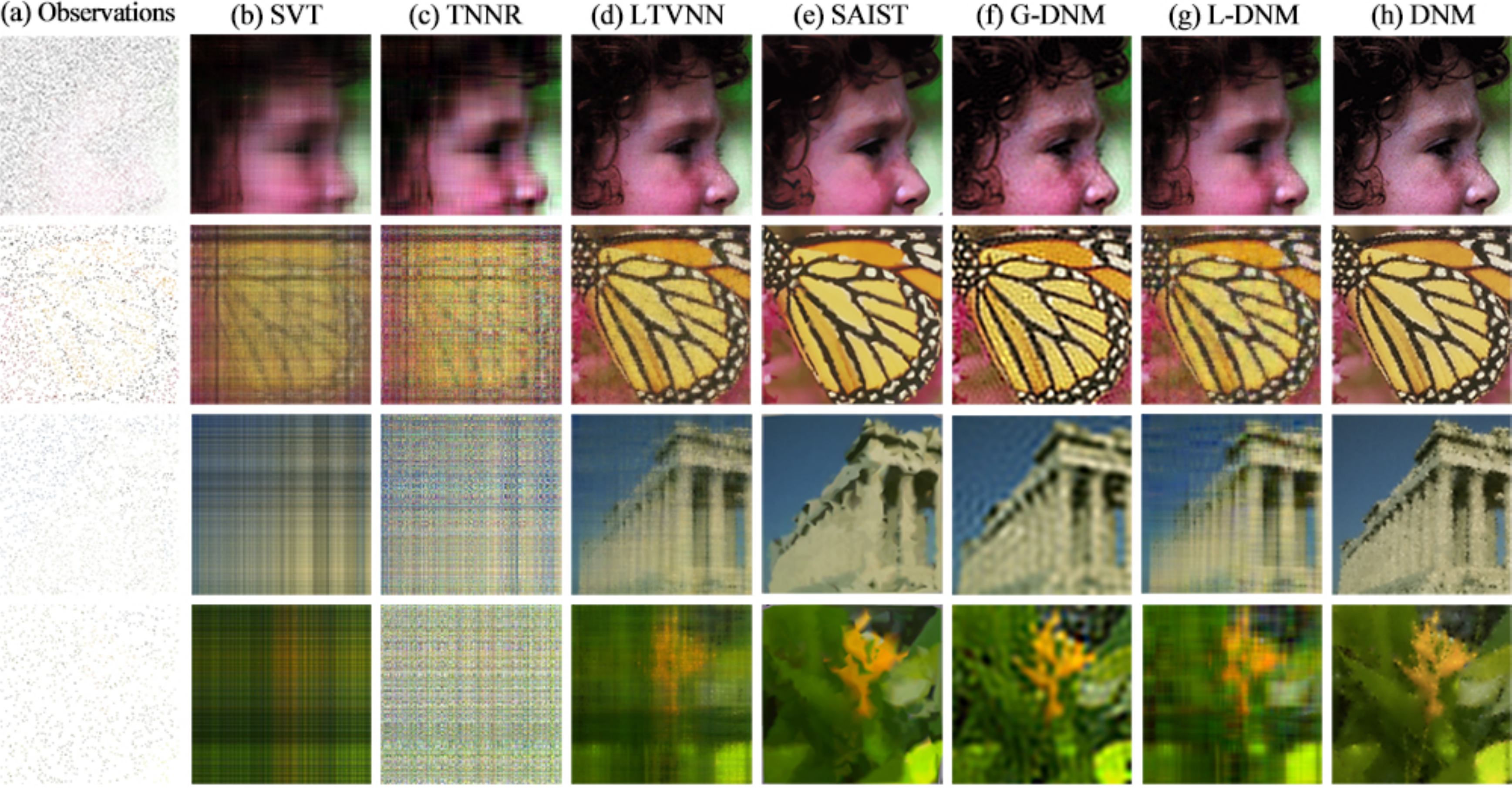}
\caption{The completion results for various corrupted images from top to bottom and their missing pixel rates: Girl (90\%), Butterfly (95\%), Parthenon (98\%), and Plants (99\%). Corrupted observations are shown in the first column and the reconstructed images using SVT, TNNR, LTVNN, SAIST, G-DNM, L-DNM, and the proposed DNM are shown from left to right. DNM results are more visually pleasing than those of previous state-of-the-art methods.}
\label{Fig:exp}
\end{figure*}

\noindent \textbf{Parameter settings.} The proposed completion algorithm has several important parameters: $p_i$, $q_i$ and $\lambda_i$. In the experiments, $p_i$ was set as 2, 8, and 512 (the size of each image) to obtain a better smoothness estimation and $q_1 = 1, q_2 = 4, q_3 = 192$, which denote the locally, blockly, and globally smooth regularizations, respectively. The weight parameter $\alpha$ was set to $10^{-3}$, $\gamma$ was set to 0.5, and $\lambda_i=\lambda = 1.5 \times 10^{-2}$ empirically. For the DNM stop conditions, we set the tolerance $\epsilon$ in Alg.\ref{Alg:Completion} to $10^{-8}$. The parameter $r$ for truncating singular values was set to 192, which is much larger than that in~\cite{tnuclearnorm}. Note that, the proposed algorithm is not sensitive to most of parameters, thus all of the parameters were set empirically which is a common setting in the context of low-rank minimization~\cite{tnuclearnorm}.

\noindent \textbf{Experimental results.} We conducted experiments using different image completion algorithms to compare the performance of the proposed DCT norm with current methods. Therefore, completion experiments were carried using the proposed DNM and state-of-the-art comparison algorithms: SVT~\cite{softthreshold}, TNNR~\cite{tnuclearnorm}, LTVNN~\cite{LTV}, and SAIST~\cite{SAIST}. Most of the comparative methods were conducted using the source code provided by authors. SVT is reported as the baseline since it is the cornerstone of the low rank approach, and TNNR is an enhanced low rank completion algorithm that is known to produce better results. LTVNN embeds the linear TV norm into the low rank minimization. SAIST is a non-local method that improves the estimated results after taking an interpolation as its initial value.

Since TNNR will be degenerated into SVT while the parameter $r$ in $||X||_r$ is set to 0, we set $r = 5$ in TNNR to compare the results of TNNR and SVT. It is interesting to note that TNNR's estimation is better than that of SVT when the missing rate $\phi$ is less than 90\%, while it is inferior to SVT when $90 < \phi$. This is because TNNR retains the largest $r$ singular values, which depict the major structure of the input image, but are most likely reflect the structure of the corrupted image with a considerable missing rate; hence, TNNR is inefficient. Since the proposed DCT norm can conductively recover image structure, we set a larger threshold $r = 192$ in DNM to avoid undesired lines caused by excessive minimization of the nuclear norm. Results are showed in Tab.\ref{Tab:PSNR} and Tab.\ref{Tab:SSIM}. The values are the average of the eight test images, where the highest evaluation result in each case is highlighted in bold. The proposed visual recovery algorithm based on the DCT norm clearly outperforms the others.

We also performed a qualitative comparison of the different completion algorithms. Results on different scenes (\emph{Girl}, \emph{Butterfly}, \emph{Parthenon}, \emph{Plants}) and missing ratios between 90\% and 99\% are shown in Fig.\ref{Fig:exp}. We also report the results exploiting the global and local DCT norms, marked as G-DNM and L-DNM in the figure. G-DNM produces some uncoordinated grids due excessive reduction of necessary high-frequency information, and the results of L-DNM are similar to that of the TV norm. Since images produced by DNM are smooth and natural, they have a better perceptual quality. It is obvious that the result of G-DNM has a clearer overall structure, while its local regions are not smooth enough. The result of L-DNM presents an opposite phenomenon. The multi-scale DNM achieves the best performance by combining them. 

The SVT results contain several lines because this algorithm slightly undermines the image structure when shrinking all singular values; therefore, the TNNR estimation is more visually pleasing when the missing pixel rate is not too high. Although the results of LTVNN and SAIST are better than those of the previous two algorithms, LTVNN is not particularly clear and SAIST over-smooths some important structures. Specifically, SAIST post-processes after initializing the image using an interpolation method, which only completes the missing pixels by exploiting neighborhood information; thus, it lacks overall structure. Furthermore, SAIST is useless when the input corrupted image does not possess the initialization produced using interpolation methods. The estimated images of the proposed algorithm are clear, sharp, and visually pleasing because the multi-scale DCT norm makes the image and its patches at different scales smooth and natural. Specifically, it obtains an estimation in which the neighborhood pixels inside are smoothed using the locally smooth DCT norm and it can also obtain an estimation that produces a globally smooth output by exploiting the globally smooth DCT norm.

\section{Discussion and conclusions}
Most existing rank-minimizing techniques do not efficiently handle data with over 90\% missing values. Therefore, we propose a powerful smooth regularization to overcome this problem: the DCT norm. Compared to the traditional TV norm, the proposed scheme involves all the pixel values and can guarantee estimation smoothness at different scales. Moreover, we demonstrate that the TV norm can be regarded as a special case of the DCT norm. By combining the truncated nuclear norm and the proposed scheme we establish an efficient image completion model. Experiments show that the estimated images using the proposed multi-scale DCT norm are more visually pleasing than those produced by the previous state-of-the-art. Additionally, the proposed smooth regularization can be independently embedded into most image processing tasks, \eg, image inpainting and image denoising.

\ifCLASSOPTIONcaptionsoff
  \newpage
\fi

\renewcommand\refname{References}
{\small
\bibliographystyle{ieee}
\bibliography{ref}

\begin{thebibliography}{10}\itemsep=-1pt

\bibitem{DCT}
N.~Ahmed, T.~Natarajan, and K.~R. Rao.
\newblock Discrete cosine transform.
\newblock {\em Computers, IEEE Transactions on}, 100(1):90--93, 1974.

\bibitem{ADMM}
S.~Boyd, N.~Parikh, E.~Chu, B.~Peleato, and J.~Eckstein.
\newblock Distributed optimization and statistical learning via the alternating
  direction method of multipliers.
\newblock {\em Foundations and Trends{\textregistered} in Machine Learning},
  3(1):1--122, 2011.

\bibitem{NLM}
A.~Buades, B.~Coll, and J.-M. Morel.
\newblock A non-local algorithm for image denoising.
\newblock In {\em CVPR}, pages 60--65, 2005.

\bibitem{softthreshold}
J.-F. Cai, E.~J. Cand{\`e}s, and Z.~Shen.
\newblock A singular value thresholding algorithm for matrix completion.
\newblock {\em SIAM Journal on Optimization}, 20(4):1956--1982, 2010.

\bibitem{nuclearnorm}
E.~J. Cand{\`e}s and B.~Recht.
\newblock Exact matrix completion via convex optimization.
\newblock {\em Foundations of Computational mathematics}, 9(6):717--772, 2009.

\bibitem{MC10}
E.~J. Cand{\`e}s and T.~Tao.
\newblock The power of convex relaxation: Near-optimal matrix completion.
\newblock {\em Information Theory, IEEE Transactions on}, 56(5):2053--2080,
  2010.

\bibitem{DC0}
W.~Chen, M.~J. Er, and S.~Wu.
\newblock Illumination compensation and normalization for robust face
  recognition using discrete cosine transform in logarithm domain.
\newblock {\em TSMC, Part B: Cybernetics}, 36(2):458--466, 2006.

\bibitem{multiregion}
H.~Cheng, Z.~Liu, N.~Zheng, and J.~Yang.
\newblock A deformable local image descriptor.
\newblock In {\em CVPR}, pages 1--8, 2008.

\bibitem{BM3D}
K.~Dabov, A.~Foi, V.~Katkovnik, and K.~Egiazarian.
\newblock Image denoising by sparse 3-d transform-domain collaborative
  filtering.
\newblock {\em TIP}, 16(8):2080--2095, 2007.

\bibitem{SAIST}
W.~Dong, G.~Shi, and X.~Li.
\newblock Nonlocal image restoration with bilateral variance estimation: a
  low-rank approach.
\newblock {\em Image Processing, IEEE Transactions on}, 22(2):700--711, 2013.

\bibitem{NCSR}
W.~Dong, L.~Zhang, and G.~Shi.
\newblock Centralized sparse representation for image restoration.
\newblock In {\em ICCV}, pages 1259--1266, 2011.

\bibitem{jointTV}
M.~Golbabaee and P.~Vandergheynst.
\newblock Joint trace/tv norm minimization: A new efficient approach for
  spectral compressive imaging.
\newblock In {\em ICIP}, pages 933--936, 2012.

\bibitem{WNNM}
S.~Gu, L.~Zhang, W.~Zuo, and X.~Feng.
\newblock Weighted nuclear norm minimization with application to image
  denoising.
\newblock In {\em CVPR}, pages 2862--2869, 2014.

\bibitem{LTV}
X.~Han, J.~Wu, L.~Wang, Y.~Chen, L.~Senhadji, and H.~Shu.
\newblock Linear total variation approximate regularized nuclear norm
  optimization for matrix completion.
\newblock {\em Abstract and Applied Analysis}, 2014.

\bibitem{tnuclearnorm}
Y.~Hu, D.~Zhang, J.~Ye, X.~Li, and X.~He.
\newblock Fast and accurate matrix completion via truncated nuclear norm
  regularization.
\newblock {\em TPAMI}, 35(9):2117--2130, 2013.

\bibitem{Gradient}
S.~Ji and J.~Ye.
\newblock An accelerated gradient method for trace norm minimization.
\newblock In {\em ICML}, pages 457--464, 2009.

\bibitem{aTVnorm}
X.~Jin, L.~Li, Z.~Chen, L.~Zhang, and Y.~Xing.
\newblock Anisotropic total variation for limited-angle ct reconstruction.
\newblock In {\em Nuclear Science Symposium Conference Record}, pages
  2232--2238, 2010.

\bibitem{DCTdistribution}
E.~Y. Lam and J.~W. Goodman.
\newblock A mathematical analysis of the dct coefficient distributions for
  images.
\newblock {\em TIP}, 9(10):1661--1666, 2000.

\bibitem{tensor13}
J.~Liu, P.~Musialski, P.~Wonka, and J.~Ye.
\newblock Tensor completion for estimating missing values in visual data.
\newblock {\em TPAMI}, 35(1):208--220, 2013.

\bibitem{SIFT}
D.~G. Lowe.
\newblock Distinctive image features from scale-invariant keypoints.
\newblock {\em International journal of computer vision}, 60(2):91--110, 2004.

\bibitem{iterTV}
S.~Osher, M.~Burger, D.~Goldfarb, J.~Xu, and W.~Yin.
\newblock An iterative regularization method for total variation-based image
  restoration.
\newblock {\em Multiscale Modeling \& Simulation}, 4(2):460--489, 2005.

\bibitem{ijcai13}
Y.~Pi, H.~Peng, S.~Zhou, and Z.~Zhang.
\newblock A scalable approach to column-based low-rank matrix approximation.
\newblock In {\em IJCAI}, pages 1600--1606, 2013.

\bibitem{TVnorm}
L.~I. Rudin, S.~Osher, and E.~Fatemi.
\newblock Nonlinear total variation based noise removal algorithms.
\newblock {\em Physica D: Nonlinear Phenomena}, 60(1):259--268, 1992.

\bibitem{ijcai15}
Z.~Shen, H.~Qian, T.~Zhou, and S.~Wang.
\newblock Simple atom selection strategy for greedy matrix completion.
\newblock In {\em IJCAI}, pages 1799--1805, 2015.

\bibitem{DCTface}
R.~Tjahyadi, W.~Liu, S.~An, and S.~Venkatesh.
\newblock Face recognition via the overlapping energy histogram.
\newblock In {\em IJCAI}, pages 2891--2896, 2007.

\bibitem{SSIM}
Z.~Wang, A.~C. Bovik, H.~R. Sheikh, and E.~P. Simoncelli.
\newblock Image quality assessment: from error visibility to structural
  similarity.
\newblock {\em TIP}, 13(4):600--612, 2004.

\end{thebibliography}
}

%
%
%

\end{document}